\documentclass[sn-basic]{sn-jnl}

\usepackage{graphicx}%
\usepackage{multirow}%
\usepackage{amsmath,amssymb,amsfonts}%
\usepackage{amsthm}%
\usepackage{mathrsfs}%
\usepackage[title]{appendix}%
\usepackage{xcolor}%
\usepackage{textcomp}%
\usepackage{manyfoot}%
\usepackage{booktabs}%
\usepackage{algorithm}%
\usepackage{algorithmicx}%
\usepackage{algpseudocode}%
\usepackage{listings}%
\usepackage{array}
\usepackage{stfloats}
\usepackage{verbatim}
\usepackage{graphicx}
\usepackage{wrapfig}
\usepackage{extarrows}
\usepackage{makecell}
\usepackage{tikz}
\usepackage{mathtools}
\usepackage{colortbl}

\newcommand*{\circled}[1]{\lower.7ex\hbox{\tikz\draw (0pt, 0pt)%
    circle (.5em) node {\makebox[0em][c]{\small #1}};}}

\def\yyh{\textcolor{black}}
\newtheorem{corollary}{Corollary}
\newtheorem{lemma}{Lemma}

\def\revise{\textcolor{black}}

\newtheorem{theorem}{Theorem}
\newtheorem{proposition}[theorem]{Proposition}%
\newtheorem{remark}{Remark}%

\newtheorem{definition}{Definition}%

\begin{document}

\title[Sanitized Clustering against Confounding Bias]{Sanitized Clustering against Confounding Bias}


\author[1,2,3,4]{\fnm{Yinghua} \sur{Yao}}\email{eva.yh.yao@gmail.com}
\equalcont{The first version of this work was done when the author was at SUSTech.} 
\author[1,2]{\fnm{Yuangang} \sur{Pan}}\email{yuangang.pan@gmail.com}
\author[1,2]{\fnm{Jing} \sur{Li}}\email{j.lee9383@gmail.com}
\author*[1,2, 4]{\fnm{Ivor} \sur{Tsang}}\email{ivor.tsang@gmail.com}
\author*[3]{\fnm{Xin} \sur{Yao}}\email{xiny@sustech.edu.cn}

\affil[1]{\orgdiv{Center for Frontier AI Research}, \orgname{Agency for Science, Technology, and Research (A*STAR)}, \orgaddress{\postcode{138632}, \country{Singapore}}}

\affil[2]{\orgdiv{Institute of High Performance Computing}, \orgname{Agency for Science, Technology, and Research (A*STAR)}, \orgaddress{\postcode{138632}, \country{Singapore}}}

\affil[3]{\orgdiv{Computer Science and Enigineering}, \orgname{Southern University of Science and Technology (SUSTech)}, \orgaddress{\city{Shenzhen}, \postcode{518055}, \state{Guangzhou}, \country{China}}}

\affil[4]{\orgdiv{Australian Artificial Intelligence Institute}, \orgname{University of Technology Sydney (UTS)}, \orgaddress{\city{Sydney}, \postcode{2007}, \state{NSW}, \country{Australia}}}



\abstract{Real-world datasets inevitably contain biases that arise from different sources or conditions during data collection. Consequently, such inconsistency itself acts as a confounding factor that disturbs the cluster analysis. Existing methods eliminate the biases by projecting data onto the orthogonal complement of the subspace expanded by the confounding factor before clustering. Therein, the interested clustering factor and the confounding factor are coarsely considered in the raw feature space, where the correlation between the data and the confounding factor is ideally assumed to be linear for convenient solutions. These approaches are thus limited in scope as the data in real applications is usually complex and non-linearly correlated with the confounding factor. This paper presents a new clustering framework named Sanitized Clustering Against confounding Bias (SCAB), which removes the confounding factor in the semantic latent space of complex data through a non-linear dependence measure. To be specific, we eliminate the bias information in the latent space by minimizing the mutual information between the confounding factor and the latent representation delivered by Variational Auto-Encoder (VAE). Meanwhile, a clustering module is introduced to cluster over the purified latent representations. Extensive experiments on complex datasets demonstrate that our SCAB achieves a significant gain in clustering performance by removing the confounding bias. The code is available at \url{https://github.com/EvaFlower/SCAB}.}




\maketitle

\section{Introduction}\label{sec1}

Clustering is an essential technique for unsupervised data analysis, whose objective is to partition samples into groups so that the samples in the same group are similar while those from different groups are significantly different~\citep{jain1999data}. 
Standard clustering methods~\citep{cheng1995mean,modha2003feature,xie2016unsupervised} is capable of capturing the desired semantic structure embedded in the clean raw data. 
However, biases are inherently present in real-world datasets, as they emerge from data collected across diverse times, scenarios, or platforms~\citep{listgarten2010correction,jacob2016correcting,li2020deep}.
These biases may introduce confounding factors that bring spurious correlation~\citep{wu2023discover}, obscuring the true underlying clustering structure~\citep{listgarten2010correction}, named as \textit{confounding biases} in this paper. 
Despite the inevitable presence of data biases, we argue that the bias information can be identified by domain experts~\citep{chierichetti2017fair,benito2004adjustment} and easily accessible (e.g., the data source usually denoted in metadata). In this study, we perform clustering \mbox{while removing the negative effect of the bias.}

Previous methods~\citep{jacob2016correcting,gagnon2012using} simply project raw data onto the subspace orthogonal to the space expanded by the confounding factor under the linear assumption before clustering.
Specifically, they decompose the data into linear combinations of the desired clustering factor and the confounding factor. In the linear space, they remove the bias information by simply subtracting the confounding covariate from the data. 
In parallel,~\citet{benito2004adjustment} applied an improved SVM which finds a linear hyperplane to separate two classes (i.e., the binary confounding factor that indicates the data source) in a supervised manner and then projects the raw data on this hyperplane. Such a method cannot scale to the scenario with multiple classes beyond the binary factor argued in~\citet{johnson2007adjusting}. In summary, former approaches are limited to the raw feature space, which may not capture high-level and representative features to describe the interested clustering factor as well as the confounding factor. In addition, they only consider linear dependence between the data and the confounding factor for clustering, which oversimplifies the real situations. The two flaws restrict the methods from applying to complex real-world data, where both the confounding factor caused by biases and the interested clustering factor are non-linearly embedded in the raw data.

\begin{figure}[!t]
    \centering
    \includegraphics[width=0.6\linewidth]{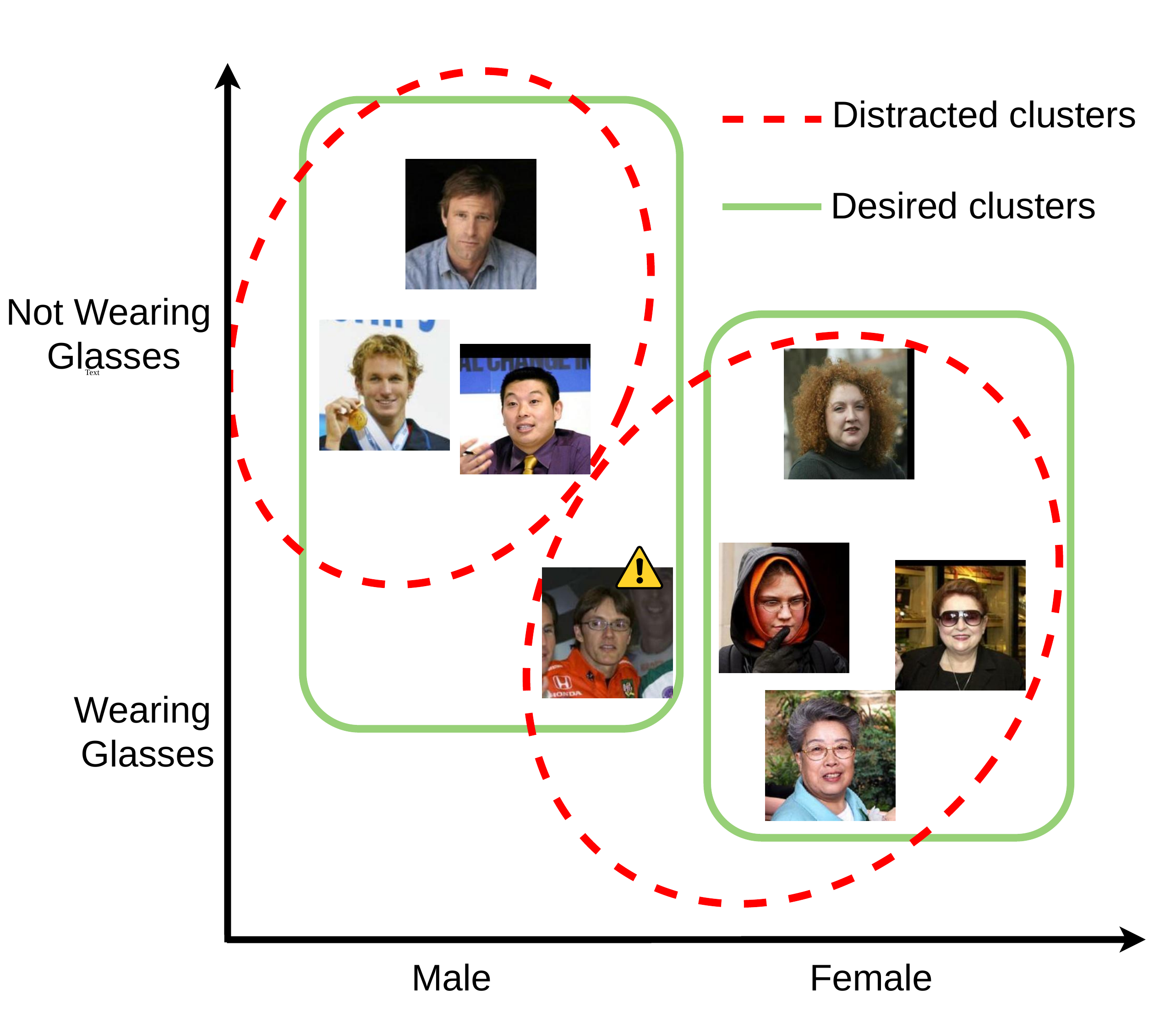}
    \caption{\label{fig:moti}Motivation: raw images contain two important factors: gender and glass. Suppose clusters with green solid lines are the desired clustering results, where the partitions are based on the gender factor only. Standard clustering algorithms that neglect the unwanted factor obtain clusters distracted by the glass factor, denoted by red dash lines. Notably, the face (a man wearing glasses) with the warning sign is viewed incorrectly grouped.}
\end{figure}
In this paper, we introduce a new clustering framework (Fig.~\ref{fig:network}), Sanitized Clustering Against confounding Bias (SCAB), applicable to high-dimensional complex biased data. SCAB is equipped with a deep representation learning module and a non-linear dependence measure to effectively eliminate bias information for superior clustering in the latent space. Specifically, our SCAB learns a clustering-favorable representation invariant to the biases within the VAE architecture~\citep{DBLP:journals/corr/KingmaW13}. The removal of bias information is achieved by minimizing the mutual information between latent representation and the confounding factor induced by biases (also interpreted as the disentanglement between the representation and the confounding factor in the later part of the paper). A tailor-designed clustering module is incorporated into VAE to cluster over the invariant representation. Benefiting from the non-linear dependence measure, our SCAB can obtain a precise clustering structure in the latent space of complex data robust to the biases. We summarize the contributions as follows:
\begin{itemize}
    \item We propose the first deep clustering framework SCAB for clustering complex data contaminated by confounding biases. Unlike existing related studies, SCAB performs semantic clustering in the latent space while minimizing the non-linear dependence between the latent representation and the biases. 
    \item Our theoretical analyses reveal that in our SCAB, (1) the loss for clustering maximizes a lower bound of the mutual information between the data representation and the desired clustering structure; (2) the loss for removing the biases minimizes an upper bound of the mutual information between the data representation and the confounding factor induced by biases.
     \item We conduct extensive experiments on seven biased image datasets. Empirical results demonstrate the superiority of our sanitized clustering with removing confounding biases, and our SCAB consistently achieves better results than existing baselines.
\end{itemize}

\section{Problem statement and related work}\label{sect:ps}
We first introduce standard clustering that neglects the data biases. Then, we motivate our problem setting where data contains confounding biases and discuss the deficiencies of existing work.
Last, we compare our setting with two related clustering branches and discuss the issues when their methodologies are applied to our setting. 

\subsection{Standard clustering}
Consider a dataset $X=[x_1, x_2, \ldots,  x_N]^T \in \mathbb{R}^{N \times D}$ consisting of $N$ samples with $D$ features. Standard clustering is to partition the dataset $X$ into $K$ groups by maximizing intra-cluster similarity and minimizing inter-cluster similarity:
\begin{equation}\label{eq:stand_cluster}
\min_{\mathcal{S}_x \in \mathbb{S}_{K,x}} \, F(\mathcal{S}_x).
\end{equation}
$\mathbb{S}_{K,x}$ denotes all feasible $K$-partitions of $X$\footnote{A $K$-partition of a set X denotes a collection of $K$ mutually disjoint non-empty subsets whose union is X. Namely, $\mathcal{S}_x=(S_1, S_2, \ldots, S_K)$, where $\bigcup_{i=1}^K S_i = X, S_i \cap S_j = \emptyset, 1 \le i  \neq j \le K$.}. $\mathcal{S}_x$ is a $K$-partition in raw feature. $F$ is the clustering objective, whose minimization aims at optimizing the quality of clustering. For instance, the $k$-means clustering objective is $F=\sum_{k=1}^{K}\sum_{n=1}^{N}s_{nk}\left\|x_n-e_{k}\right\|_{2}^{2}$, where $e_k$ is the $k$-th cluster centroid. $s_{nk}\in \{0,1\}$ denotes the cluster assignment which equals $1$ if $x_n$ is assigned to the $k$-th cluster and $0$ otherwise.

While classical approaches~\citep{cheng1995mean} conduct clustering in the raw feature space, recent deep clustering methods~\citep{xie2016unsupervised,guo2017improved,DBLP:conf/cvpr/HuangGZ20,niu2022spice} explores clustering-favourable latent representation for a better structure discovery. However, when there are obvious variances resulting from biases present in the data, all standard clustering methods are unavoidably distracted by the confounding biases and the clustering performance will degenerate (see Table~\ref{tb:simple},~\ref{tb:complex}).

\subsection{Clustering data contaminated by confounding biases}
When the data is collected from multiple sources or different conditions, each source may have its own biases. These biases could mask genuine similarities or differences between data points, distorting the desired clustering results~\citep{jacob2016correcting}. In this case, the data source can be said a confounding factor that hinders the accurate clustering structure.
In addition, confounding factors that bias the clustering results in other scenarios can also be identified by the domain experts. For instance, in the facial recognition task, whether people wearing glasses or not would impair the recognition results for identity~\citep{sharif2016accessorize}. 

In order to deliver a precise clustering structure, we consider removing the influence of these confounding biases. We suppose such bias information can be always described by a label indicator, which is an effective encoding for the confounding factor (e.g., a source indicator indicating the data is from source 1, 2, or etc.). Given the complete instance-wise confounding factor, we define our problem setting in the following. 
\begin{definition}[Sanitized clustering with the removal of confounding bias]\label{df:our problem}
Let $X \in \mathcal{R}^{N \times D}$ denote a dataset consisting of $N$ samples. Let $C=[c_1, c_2, \ldots, c_N]^T \in \left\{0, 1\right\}^{N \times G}$ be the corresponding labels with regards to a certain confounding factor~$c$, where $C_{i, j}=1$ if $x_i$ belongs to class $j$ and $C_{i, j}=0$ otherwise; $G$ is the number of categories. Our goal is to find a partition  $\mathcal{S}_x \in \mathbb{S}_{K,x}$, such that $\mathcal{S}_x$ is uninformative of~$c$. The objective is:
\begin{equation}\label{eq:scab}
 \min_{\mathcal{S}_x \in \mathbb{S}_{K,x}} \, F(\mathcal{S}_x),
 \quad s.t. \; \mathcal{S}_x \perp c,
\end{equation}
where $\perp$ denotes that two variables are independent. 
\end{definition}

\textbf{Existing work.} Some work~\citep{jacob2016correcting,listgarten2010correction,gagnon2012using} targeting the problem (Definition~\ref{df:our problem}) are built on a linear model that assumes the confounding factor is linearly correlated with the data. Mathematically, let $A \in \{0,1\}^{N \times K}$ denote a group assignment matrix, and each raw of $\alpha \in \mathbb{R}^{K \times D}$ denote a cluster centroid.  Supposing $C \in \{0,1\}^{N \times G}$ represents the class matrix converted via the confounding factor~$c$, and each raw of $\beta \in \mathbb{R}^{G \times D}$ denotes the centroid of the corresponding category. Then, the linear model is formulated as:
\begin{equation}\label{eq:ruv}
    X=A\alpha+C\beta+\varepsilon,
\end{equation}
where $\varepsilon$~denotes some prior noise. $\beta$ can be estimated via a regression model by setting $A\alpha=0$~\citep{jacob2016correcting}. 
By subtracting the bias term~$C\beta$, a purified dataset $\hat{X}$ is: 
\begin{equation}\label{eq:lp}
   \hat{X} = X-C\beta.
\end{equation}
Then, a regular clustering method like $k$-means is conducted on $\hat{X}$ to obtain a partition $\mathcal{S}_{\hat{x}}$ (i.e., $A$ and $\alpha$). Under the linear assumption, the obtained partition thus satisfies the independent constraint, namely, $\mathcal{S}_{\hat{x}} \perp c$.

\textit{Deficiencies that make existing approaches impractical for high-dimensional complex data.} (1) They are developed in the raw feature space, which is insufficient to discover the underlying structures in terms of the interested factor as well as the confounding factor, i.e., $\alpha$ and $\beta$ in~Eq.\eqref{eq:ruv}. (2) Only linear dependence is explored. The removal  of the confounding factor is simply via a linear projection, i.e., Eq.\eqref{eq:lp}, which will fail when the data has a non-linear dependence with the confounding factor. 

\subsection{Related clustering branches}
\label{sect:ac&fc}

\textbf{Alternative clustering}~\citep{wu2018iterative} suggests finding an alternative structure based on the existing clustering result to unveil a new perspective of the dataset. 
\citet{niu2013iterative,wu2019solving} pursued a novel clustering while minimizing its dependence on the given clustering structure. 
In particular, the relevance is measured by a specific kernel independence measure, the Hilbert-Schmidt independence criterion (HSIC). 
Given a dataset $X\in \mathcal{R}^{N \times D}$, let $Y=[y_1, y_2, \ldots, y_N]^T \in \left\{0, 1\right\}^{N \times K_0}$ be an existing clustering result over $X$, where $K_0$ is the cluster number. $y_{i, j}=1$ if $x_i$ belongs to the $j$-th cluster and $y_{i, j}=0$ otherwise. The aim is to obtain an alternative structure $U \in \mathcal{R}^{N\times K}$ with $K$ clusters on a subspace with lower dimensions $Q$ $\ll D$. The objective is defined as:
\begin{align} \label{eq:alter_cluster}
    \max_{W, U} \, \operatorname{HSIC}(X W, U)-\lambda \operatorname{HSIC}(X W, Y), \quad s.t. \quad W^{T} W=I, U^{T} U=I.
\end{align}
where $W \in \mathcal{R}^{D\times Q}$ denotes the projection matrix. The solution of~Eq.\eqref{eq:alter_cluster} can be referred to~\citet{niu2013iterative,wu2018iterative}.

\textit{Alternative clustering vs. our setting (Def.~\ref{df:our problem}).} Although starting from a different motivation, Eq.\eqref{eq:alter_cluster} can be a practical implementation form for~Eq.\eqref{eq:scab} by replacing the given clustering structure with the confounding factor. However, obtaining the subspace irrelevant to the confounding factor by a linear projection is not suitable for the high-dimensional complex dataset where the factor is a high-level semantic feature. Meanwhile, such a technique requires storing a full batch of data for clustering, which incurs a heavy memory complexity of $\mathcal{O}(N^2)$.

\yyh{\textbf{Fair clustering}\footnote{\yyh{Note that some recent work~\citep{mahabadi2020individual,vakilian2022improved} which are also called fair clustering are not related to our setting, because they follow the individual fairness~\citep{jung2019center} where group attributes are not specified.}} that extends group fairness~\citep{feldman2015certifying} to clustering} explores the clustering structure while ensuring a balanced proportion within each cluster regarding some specified sensitive attribute~\citep{chierichetti2017fair}. 
With a slight abuse of notation, suppose $X$ can be represented as the disjoint union of $H$ protected subgroups in terms of some sensitive attribute $a$, i.e., $X=\bigsqcup_{h \in [H]} X_{h}=\bigcup_{h \in [H]}\{(x, h)\mid x \in X_h\}$.
For a clustering result $\mathcal{S}_{x} \in \mathbb{S}_{K,x}$, the balance of each cluster $S_k$ and the whole clustering result $\mathcal{S}_{x}$ can \mbox{be respectively defined as:}
\begin{align}\label{eq:balance}
\mathcal{B}(S_{k} \mid a) =\min_{h \neq h' \in [H]} \frac{\lvert X_{h} \cap S_{k} \rvert}{\lvert X_{h^{\prime}} \cap S_{k}\rvert} \in[0,1];  \quad \mathcal{B}(\mathcal{S}_{x} \mid a) = \min_{k\in [K]}\mathcal{B}\left(S_{k} \mid a\right).
\end{align}
The higher the balance of each cluster, the fairer the clustering result will be. A $(T, K)$-fair clustering is defined as:
\begin{equation}\label{eq:fair_cluster}
\min_{\mathcal{S}_{x}\in \mathbb{S}_{K,x}} F(\mathcal{S}_{x}),  \quad
    s.t. \; \mathcal{B}\left(\mathcal{S}_{x} \mid a\right) \geq T,
\end{equation}
where $T$ controls the degree of fairness for clustering. 
Eq.\eqref{eq:fair_cluster} pursues a partition where each cluster approximately maintains the same ratio over the sensitive attribute as that in the whole dataset~\citep{chierichetti2017fair,kleindessner2019guarantees}.  

\textit{Fair clustering vs. our setting (Def.~\ref{df:our problem}).} Both fair clustering and our problem setting require information about some specific attribute (factor) before conducting clustering. However, fair clustering aims to deliver a clustering structure that meets fairness criteria over a certain sensitive attribute. The clustering performance would degrade when imposing such an extra fairness constraint~\citep{chierichetti2017fair}. In contrast, our target is to improve clustering by eliminating the effect of the confounding factor that distracts the clustering results. Therefore, fair clustering methods (Eq.\eqref{eq:fair_cluster}) cannot be applied to our setting, except a recent deep fair clustering (DFC)~\citep{li2020deep}. DFC was proposed to learn fair representation for clustering and claimed to adopt stronger fairness criteria than the balance criteria (Eq.\eqref{eq:balance}). It introduced an adversarial training paradigm in the context of deep standard clustering to encourage clustering structures to be independent of the sensitive attribute. 
This form of fair clustering objective is the same as ours~(Eq.\eqref{eq:scab}) when the sensitive attribute is designated as the confounding factor. However, the adversarial training increases the difficulty of model training and requires an extra complex constraint to maintain the clustering structure.   

\section{Sanitized clustering against confounding bias}
\label{sect:scab}
This section presents a new framework SCAB to deliver desired clustering structures on complex datasets contaminated by confounding biases. 

\subsection{Deep semantic clustering in the latent space}
We perform clustering in the latent space to capture the semantic structure of complex data.
Consider a general task (e.g., data reconstruction) that involves encoding the data~$x$ into its representation~$z$ via the posterior $q(z\mid x)$. The objective of deep semantic clustering includes the objective $L$ for representation learning and the objective $F$ for clustering on the representations~\citep{xie2016unsupervised,boubekki2021joint}. Namely,
\begin{equation}\label{eq:deep_cluster}
    \min_{q, \; \mathcal{S}_z \in \mathbb{S}_{K, z}} L(q, x)+\eta F(\mathcal{S}_z).
\end{equation}
$\mathcal{S}_z$ denotes a partition in the space where $z$ resides. $\mathbb{S}_{K, z}$ is defined similarly as $\mathbb{S}_{K, x}$ in Eq.\eqref{eq:stand_cluster}. $\eta$ is a trade-off parameter that balances representation learning and clustering.

\yyh{In particular, we choose Variational AutoEncoder (VAE)~\citep{DBLP:journals/corr/KingmaW13} to compute $L(q, x)$, because VAE includes modeling of $q(z\mid x)$, and VAE based clustering can obtain good clustering-favorable representations and is effective for various complex datasets~\citep{jiang2017variational}. }

\subsection{\mbox{Clustering on representations invariant to confounding factor}}
Eq.\eqref{eq:deep_cluster} conducts semantic clustering without considering the existence of the confounding bias. To eliminate the negative impact of the bias on the target clustering structure $\mathcal{S}_z$, we propose deep semantic clustering independent of the confounding factor~$c$. Recalling~Eq.\eqref{eq:scab}, our objective is formulated as:
\begin{align} \label{eq:deep scab}
    \min_{q, \; \mathcal{S}_z \in \mathbb{S}_{K, z}} L(q, x)+\eta F(\mathcal{S}_z), 
 \quad s.t. \; \mathcal{S}_z \perp c.
\end{align}
Since a partition $\mathcal{S}_z $ is defined over the whole dataset while $c$ is collected per sample, directly implementing $\mathcal{S}_z \perp c$ is complex and incurs large computational costs. Instead, we impose an alternative independence constraint between the sample representation $z$ and the confounding factor~$c$, i.e., $z \perp c$, \mbox{both of which are defined at the sample level.}

\begin{proposition}\label{thm:zc}
Let $\mathcal{Z}$ be the representation space, and $Z=\{z_1, z_2, \ldots, z_N\}^T \in \mathcal{Z}$ be the representation set of the dataset $X$. Suppose the clustering algorithm $\mathcal{A}$ takes $Z$ as an input and returns a partition $\mathcal{S}_z$ of $Z$. Namely, $\mathcal{A}: Z \longrightarrow \mathcal{S}_z$. If $z \perp c$, then we naturally have $\mathcal{S}_z \perp c$.
\end{proposition}

Proposition~\ref{thm:zc} demonstrates clustering over representations~$z$ that is invariant to the confounding factor~$c$ can derive a clustering structure~$\mathcal{S}_z $ that is uninformative of the confounding factor~$c$.
Thus, our objective can be reformulated as:
\begin{align} \label{eq:scab_1}
    \min_{q, \; \mathcal{S}_z \in \mathbb{S}_{K, z}} L(q, x)+\eta F(\mathcal{S}_z), \quad
    s.t. \; z \perp c.
\end{align}
The independence constraint $z \perp c$ is still a strong condition and is difficult to optimize directly. We approximate it by minimizing the mutual information $I(z,c)$~\citep{moyer2018invariant}. Adding the term $I(z,c)$, the objective~Eq.\eqref{eq:scab_1} becomes:
\begin{equation}\label{eq:scab_2}
\begin{aligned}
    \min_{q, \; \mathcal{S}_z \in \mathbb{S}_{K, z}} &L(q, x)+\eta_1 I(z, c) + \eta_2 F(\mathcal{S}_z).
\end{aligned}
\end{equation}
where $\eta_1$ and $\eta_2$ are the hyper-parameters that balance the three losses. In Eq.\eqref{eq:scab_2}, the interested clustering factor, which is embedded in the representation~$z$, and the confounding factor~$c$ can be semantically described in the latent space~\citep{xie2016unsupervised,vincent2010stacked}. Meanwhile, these two factors are disentangled in the latent space. 
By optimizing~Eq.\eqref{eq:scab_2}, we can obtain a semantic clustering structure $\mathcal{S}_z$ \mbox{that is irrelevant to the confounding factor~$c$.}

\begin{figure}[!t]
    \centering
    \includegraphics[width=\linewidth]{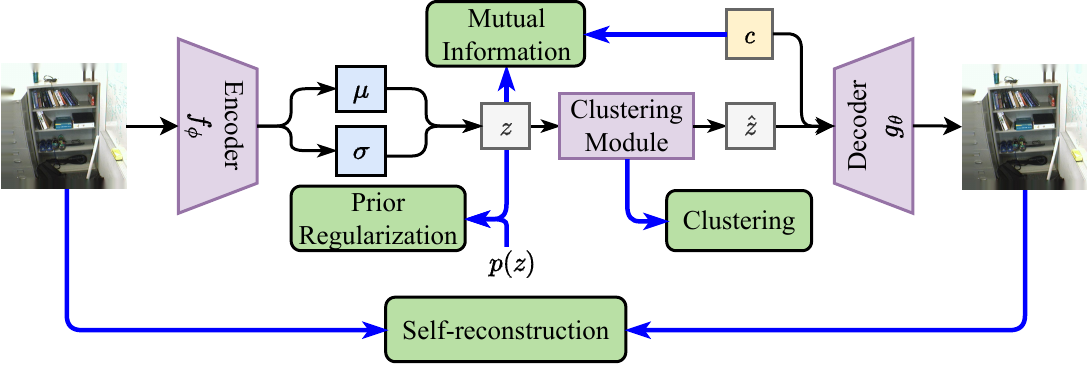}
    \caption{\label{fig:network}The architecture of our Sanitized Clustering against confounding Bias (SCAB).}
\end{figure}
\subsection{The overall clustering framework: SCAB}
To summarize, our framework jointly trains with three modules. First, the VAE structure is adopted as the feature extractor module for learning semantic features. Further, we introduce one disentangling module over the latent space derived by VAE, to disentangle the confounding factor~$c$ and other salient information~$z$ encoded in the data (i.e., $z \perp c$). Last, a clustering module  based on soft $k$-means is incorporated within the VAE structure to perform clustering on the factor of interest (embedded in $z$) only. 

\subsubsection{Variational autoencoder}
Accordingly, we can formulate the statistical (non-linear) dependence between $x$ and $c$ in the latent space, i.e., $p(x, z, c)=p(z, c) p(x \mid z, c)$ where $z$ is the latent variable of $x$.

Similar to VAE~\citep{DBLP:journals/corr/KingmaW13}, the variational lower bound for the expectation of conditional log-likelihood  $\mathbb{E}_{(x, c)}\left[\log p(x \mid c)\right]$ can be deduced as follows:
\begin{align}\label{eq:elbo}
\mathbb{E}_{(x, c)}\left[\log p(x \mid c)\right] &\geq  \mathbb{E}_{(x, c)}\big[\mathbb{E}_{z \sim q(z \mid x)}[\log p(x \mid z, c)] -KL[q(z\mid x) \| p(z)]\big].
\end{align}
The conditional decoder~$p(x \mid z, c)$ takes both $z$ and $c$ as input. We simplify the distribution of $z$ to solely depend on the input~$x$, optimized by the encoder~$q(z \mid x)$. $p(z)$ is the prior distribution that is defined as a Gaussian noise.

We parameterize the approximate posterior $q(z \mid x)$ with an encoder~$f_{\phi}$ that encodes a data sample $x$ to its latent embedding $z$, and parameterize the likelihood $p(x \mid z, c)$ with a conditional decoder $g_{\theta}$ that produces a data sample conditioned both on the latent embedding $z$ and the observed confounding factor $c$. Usually, a particle $z_n$ is sampled from $q(z \mid x)$ for reconstructing $x_n$~\citep{DBLP:journals/corr/KingmaW13}. Then, the loss function (minimization) based on the Monte Carlo estimation of the variational lower bound in~Eq.\eqref{eq:elbo} is defined as:
\begin{align}\label{eq:vae}
\mathcal{L}_{\text{VAE}} & =\sum_{n=1}^{N}\ell_{\text{r}}\left(x_n, g_{\theta}(z_n, c_n)\right) +\sum_{n=1}^{N}K L[q_{\phi}(z \mid x_n) \| p(z)],
\end{align}
where $\ell_{\text{r}}$ denotes the reconstruction loss, which can be instantiated with mean squared loss or cross-entropy loss. $\mathcal{L}_{\text{VAE}}$ is used to calculate the first term $L(q,x)$ in Eq.\eqref{eq:scab_2}.

\subsubsection{Disentanglement by minimizing mutual information}

By minimizing the mutual information between the latent variable~$z$ and the confounding factor~$c$, the bias information is disentangled from other salient information in the latent space. 

\begin{lemma}[MI upper bound~\citep{moyer2018invariant}]
The mutual information between the latent representation~$z$ and the confounding factor~$c$, i.e.,  $I(z, c)$, is subject to an upper bound:
\begin{align}\label{eq:ub_mi}
    I(z, c) \leq -H(x \mid c)-\mathbb{E}_{x, c, z \sim q}[\log p(x \mid z, c)]+\mathbb{E}_{x}[K L[q(z \mid x) \| q(z)]].
\end{align}
\end{lemma}
As $I(z, c)$ is not directly computable, we use its upper bound (Eq.\eqref{eq:ub_mi}). The constant $H(x \mid c)$ can be ignored. The second term is a reconstruction loss (Eq.\eqref{eq:vae}). The third term on the right of Eq.\eqref{eq:ub_mi} is intractable to compute and is approximated by the pairwise distances as follows~\citep{moyer2018invariant}: 
\begin{align}
\mathbb{E}_{x}[K L[q(z \mid x) \| q(z)]] \approx \sum_{x} \sum_{x^{\prime}} {K L\left[q(z \mid x) \| q(z \mid x')\right]}. \nonumber
\end{align}
The loss function is finally defined as:
\begin{align}\label{eq:mi}
    \mathcal{L}_{\text{MI}} = \sum_{n=1}^{N}\ell_{\text{r}}\left(x_n, g_{\theta}(z_n, c_n)\right) +\sum_{n=1}^{N}\sum_{m=1}^{N}K L[q_{\phi}(z \mid x_n) \| q_{\phi}(z \mid x'_{m})].
\end{align}
The minimization of $I(z, c)$, the second term in Eq.\eqref{eq:scab_2}, is thus replaced by the minimization of its upper bound, i.e.,  $\mathcal{L}_{\text{MI}}$.

\subsubsection{Clustering over the $c$-invariant embedding}
\label{sect:cluster}
Eq.\eqref{eq:mi} helps to filter out the information of the confounding factor~$c$ from the latent code $z$. For the sake of efficiency, we apply $k$-means algorithm to conduct clustering on the $c$-invariant embedding $z$. 
Particularly, the $k$-means clustering loss is defined as:
\begin{equation}\label{eq:cluster}
\mathcal{L}_{\text{cluster}} = \sum_{n=1}^{N}\sum_{k=1}^{K} s_{n k}\left\|z_n-e_{k}\right\|_{2}^{2}.
\end{equation}
$\mathcal{L}_{\text{cluster}}$ is used to compute the third term $F(\mathcal{S}_z)$ in Eq.\eqref{eq:scab_2}. $\mathbf{e}=\{e_1, e_2, \ldots, e_K\}$ are the collection of $K$ centroids. $s_{nk} \in \{0, 1\}$ refers to the group assignment that assigns the latent embedding $z$ to its closest clustering centroid. Namely,
\begin{equation}\label{eq:delta_final}
\lambda_{nk} =\frac{\exp \left(-\tau\left\|z_n-e_{k}\right\|_{2}^{2}\right)}{\sum_{i=1}^{K} \exp \left(-\tau\left\|z_n-e_{i}\right\|_{2}^{2}\right)}, \quad s_{nk} =\left\{\begin{array}{ll}
1 & k=\operatorname{argmax}_{j} \lambda_{nj} \\
0 & \text { otherwise }
\end{array}\right.,
\end{equation}
where $k=1, 2, \ldots, K$. $\tau$ is the temperature and is set to $5$ in the experiment.

Due to the reconstruction loss in VAE (Eq.\eqref{eq:vae}), the latent representations would contain many sample-specific details, which is detrimental to clustering. We follow~\citep{pan2021streamlining} to introduce the following skip-connection formulation to unify the reconstruction goal and the clustering goal. Namely,
\begin{equation}\label{eq:z_hat_final}
    \hat{z}_n=h_{\psi}\left(z_n, \tilde{z}_n\right), \ \textrm{where} \ \tilde{z}_n=\sum_{k=1}^{K} s_{nk} e_{k}.
\end{equation}
\yyh{Note that $\tilde{z}_n$ is one of $K$ clustering centroids as $s_{nk}$ is a one-hot assignment. $h_{\psi}$ constructs a new latent representation~$\hat{z}_n$ that incorporates not only the original $c$-invariant embedding $z_n$ but also its belonging clustering centroid $\tilde{z}_n$ as the input of the decoder. $h_{\psi}$ is implemented as a linear layer.}

\subsubsection{Objective and optimization of SCAB}
Integrating all three modules comes to our new framework Sanitized Clustering Against confounding Bias (SCAB) (Fig.~\ref{fig:network}). Its final objective is formulated as:
\begin{align}\label{eq:obj}
\mathcal{L}(\Theta, \mathbf{e}) =& \mathcal{L}_{\text{VAE}}+\eta_1\mathcal{L}_{\text{MI}}+\eta_2\mathcal{L}_{\text{cluster}},
\end{align}
where $\Theta=\{\theta, \phi, \psi\}$ denote the network parameters and $\mathbf{e}$ represent clustering parameters.  $\eta_1$ and $\eta_2$ are the trade-off parameters. 

\textit{Clustering structure}. After training the model, the  clustering structure $\mathcal{S}_{z}=(S_1, S_2, \ldots, S_K)$ is calculated by: $S_k = \{z_n \mid s_{nk}=1, n=1,2,\ldots,N\}$, where $k=1,2,\ldots,K$ and $s_{nk}$ is defined in Eq.\eqref{eq:delta_final}.

In Eq.\eqref{eq:obj}, two types of parameters, i.e., network parameters $\Theta$, and clustering parameters $\mathbf{e}$, are coupled together, which hinders them from joint optimization. We adopt coordinate descent to alternatively optimize $\Theta$ and $\mathbf{e}$. 

To make our SCAB scalable to large-scale problems, we adopt stochastic gradient updates for all parameters. However, such an update for clustering centroids $\mathbf{e}$ would be unstable because the clustering centroids estimated by different mini-batch data may be of great discrepancy. To overcome this issue, we apply the exponential moving average (EMA) update for the centroids since the EMA update yields good stability~\citep{van2017neural}. Specifically, each centroid $e_k$ is updated online using the assigned neighbor representations in the mini-batches~$\{z_b\}_{b=1}^B$:
\begin{equation}\label{eq:ema}
\mu_k^{(t)}:=\gamma\mu_k^{(t-1)} + (1-\gamma)\sum_{b=1}^{B} s_{bk}^{(t-1)} z_b^{(t-1)},\;   B_k^{(t)}:=\gamma B_k^{(t-1)}+(1-\gamma)\sum_{b=1}^{B}s_{bk}^{(t-1)}, \;  e_k^{(t)} := \frac{\mu_k^{(t)}}{B_k^{(t)}},
\end{equation}
\mbox{where $\gamma \in [0,1]$ is a decay parameter (set to 0.995 by default). $t$ is the iteration index.}

\subsection{Theoretical analysis}
In this section, we theoretically analyze that optimizing network parameters~$\Theta$ of SCAB in Eq.\eqref{eq:obj} is equivalent to (1) maximizing the lower bound of the mutual information between the representation and the interested clustering structure, i.e., $\max_z I(z, s)$, while (2) minimizing the upper bound of the mutual information between the representation and the confounding factor, i.e., $\min_z I(z, c)$.

\begin{theorem}\label{thm:z_e}
Assume a fixed clustering structure, i.e., the clustering centroids $\mathbf{e}=\{e_1, e_2, \ldots, e_K\}$ and the cluster assignments $\{s_{n}\}_{n=1}^{N}$, where $s_n$ is a $K$-dimensional one-hot vector and $s_{nk}$ is defined in Eq.\eqref{eq:delta_final}. The minimization of our clustering object $\mathcal{L}_{\text{cluster}}$ is equivalent to maximizing the lower bound of the mutual information between the representation $z$ and the interested clustering structure, represented by the group assignment $s$, i.e., $I(z, s)$, \revise{given the clustering centroids $\mathbf{e}$}.
\end{theorem}

\begin{proof}
Based on the definition of mutual information, we have
\begin{equation*}
    I(z,s) = \int p(z, s) \log \frac{p(z, s)}{p(z) p(s)}dz ds =\int p(z, s) \log \frac{p(s \mid z)}{p(s)}dz ds.
\end{equation*}    
Assume $p(x, c, z, s)=p(x,c)p(z \mid x,c)p(s \mid x,c,z)=p(x,c)p(z \mid x,c)p(s \mid x,c)$, where $p(s \mid x,c,z)=p(s \mid x,c)$ follows the conditional independence.
Since $p(s \mid z)=\int p(x, c, s \mid z)dxdc=\int \frac{p(z \mid x, c)p(x, c)}{p(z)}p(s \mid x,c)dxdc$ is intractable, we introduce an auxiliary distribution~$q(s \mid z)$ as an approximation to $p(s \mid z)$~\citep{DBLP:conf/iclr/AlemiFD017}. Because
$\mathrm{KL}[p(s \mid  z)||q(s \mid z)] \geq 0 \Longrightarrow  \int p(s \mid z) \log p(s \mid z) ds \geq \int p(s \mid z) \log q(s \mid z)ds$,
we obtain 
\begin{align}
I(z, s) \geq & \int p(z, s) \log \frac{q(s \mid z)}{p(s)}dz ds =\int p(z, s) \log q(s \mid z) dz ds+H(s) \\
\overset{\circled{1}}{=}  & \int  p(x,c)p(z \mid x,c)p(s \mid x,c) \log q(s \mid z) dx dc dz ds + H(s) \nonumber\\
= & \mathbb{E}_{(x, c) \sim p(x,c)}\mathbb{E}_{z \sim p(z\mid x, c)} \mathbb{E}_{s \sim p(s \mid x,c)}\log q(s \mid z) ds + H(s) = L_{I}(z, s) + H(s).\nonumber
\end{align}
$\circled{1}$ is valid since $p(z, s)=\int p(x,c,z,s)dxdc=\int p(x,c)p(z \mid x,c)p(s \mid x, c)dxdc$. 

The auxiliary distribution~$q(s\mid z)$ can be naturally defined by our $k$-means clustering module (Section~\ref{sect:cluster}). Accordingly, we have 
$q(s_{nk}=1 \mid z_n) = \frac{\exp \left(-\tau\left\|z_n-e_{k}\right\|_{2}^{2}\right)}{\sum_{i=1}^{K} \exp \left(-\tau\left\|z_n-e_{i}\right\|_{2}^{2}\right)}$.
Note that the posterior~$p(z \mid x,c)$ is approximated by the VAE encoder $q(z \mid x)$ constrained with the minimization of $I(z,c)$ and usually one particle $z_n$ is sampled from $q(z|x)$ to reconstruct $x_n$~\citep{DBLP:journals/corr/KingmaW13}. Together with the given cluster assignment $s_n \sim p(s\mid x,c)$, we have
\begin{equation*}
L_{I}(z,s) = \sum_{n=1}^{N}\sum_{k=1}^{K}s_{nk} \log \frac{\exp \left(-\tau\left\|z_n-e_{k}\right\|_{2}^{2}\right)}{\sum_{i=1}^{K} \exp \left(-\tau\left\|z_n-e_{i}\right\|_{2}^{2}\right)} \xLongrightarrow[\tau \longrightarrow +\infty]{\circled{1}}-\sum_{n=1}^{N} \sum_{k=1}^{K} s_{n k}\left\|z_n-e_{k}\right\|_{2}^{2}.
\end{equation*}
$\circled{1}$ is valid because the value of $q(s_{nk}=1 \mid z_n)$ approaches zero for all $k$ except for the one corresponding to the smallest distance~\citep{kulis2012revisiting}. 
Then, we have
\begin{equation}\label{eq:z_s}
    I(z, s) \geq -\mathcal{L}_{\text{cluster}} + H(s).
\end{equation}
$H(s)$ can be ignored since it is a constant. We complete the proof.
\end{proof}

\begin{corollary}\label{corollary:1}
Fixing the centroids $\mathbf{e}$ as well as the cluster assignments $\{s_n\}_{n=1}^{N}$, Eq.\eqref{eq:obj} is subject to the following lower bound:
\begin{equation}\label{eq:obj_mi}
    Eq.\eqref{eq:obj}\geq -\mathbb{E}_{(x, c)}[\log p(x  \mid c)]+\eta_1 I(z, c)-\eta_2 I(z, s).
\end{equation}
Because three terms of Eq.\eqref{eq:obj} are respectively lower bounded according to  Eq.\eqref{eq:elbo} $\&$ Eq.\eqref{eq:vae},  Eq.\eqref{eq:ub_mi} $\&$ Eq.\eqref{eq:mi}, and Eq.\eqref{eq:z_s}. 
\end{corollary}
From Corollary~\ref{corollary:1}, we conclude that the optimization for $\Theta$ given $\mathbf{e}$ is to learn a clustering-favorable representation, which is invariant to the confounding factor~$c$.

\revise{
\begin{remark}[Continuous/Incomplete confounding factor]
(1) Our method and theoretical analysis are applicable to the continuous confounding factors as well, as they do not specify the exact form of the confounding factor. We will conduct experiments to demonstrate the efficacy of our SCAB on the continuous confounding factor in Section~\ref{sect:exp}.
(2) For the known confounding factor without ready-to-use annotations, we additionally collect a small amount of supervision for it to avoid too much manual cost. Then, we can solve the problem in a semi-supervised manner, which will be explored in Section~\ref{sect:semi}. 
\end{remark}
}

\section{Experiments}
\label{sect:exp}

\textit{Dataset.} \revise{We conduct experiments on six image datasets (\textit{UCI-Face, Rotated Fashion, MNIST-USPS, Office-31, CIFAR10-C, Rotated Fashion-Con}) and one signal-vector dataset (\textit{HAR}) containing confounding factors that would bias the clustering results (see Table~\ref{tb:dataset})}. In particular, \textit{Rotated Fashion} is constructed by introducing the rotation factor into the Fashion-MNIST dataset. Specifically, we pick up images from cloth categories, i.e., ``T-shirt/top'', ``Trouser'', ``Pullover'', ``Dress'', ``Coat'' and ``Shirt'', for simplicity. We first randomly sample $1,000$ images from each of the six classes (zero degree). Then, each image is augmented with four views of $72$, $144$, $216$, and $288$ degrees, respectively. For \textit{Office-31}, we select samples from Amazon and Webcam as training data following~\citet{li2020deep}. \revise{\textit{Rotated Fashion-Con} is contructed similarly, but the rotation angle is set to a continuous range of 0 to 60 degrees.} For \textit{CIFAR10-C}, we consider one in each main category of corruptions, namely, frost, Gaussian blur, impulse noise, and elastic transform for simplicity.
\begin{table}[!t]
\centering
\caption{\label{tb:dataset} Statistics of datasets. $K$ denotes the number of clusters. $G$ denotes the number of categories or \revise{range of the values.}}
\renewcommand{\arraystretch}{1}
\setlength{\tabcolsep}{1.1mm}{
\begin{tabular}{c|c|c|c|c}
\bottomrule[1.3pt]
Dataset               & \#sample & \#dim                 & $K$ & confounding factor ($G$)    \\ \hline
\textit{UCI-Face}~\citep{DBLP:journals/sigkdd/BayKPS00}              & 1,872    & $32\times30$          & 4         & identity (20)       \\
\textit{Rotated Fashion}~\citep{xiao2017fashion}       & 30,000   & $28\times28$          & 5         & cloth category (6)  \\
\textit{MNIST-USPS}~\citep{726791,hull1994database}            & 67,291   & $32\times32$          & 10        & source of digit (2)     \\
\textit{Office-31}~\citep{saenko2010adapting}             & 3,612    & $224\times224\times3$ & 31        & domain source (2)   \\
\textit{CIFAR10-C}~\citep{DBLP:conf/iclr/HendrycksD19}             & 40,000   & $32\times32\times3$   & 10        & corruption type (4) \\
\textit{HAR}~\citep{anguita2013public}              & 10,299    & $561$          & 6         & subject (30)     \\
\textit{Rotated Fashion-Con}~\citep{xiao2017fashion}       & 30,000   & $28\times28$         & 6         & rotation angle (0-60) \\
\bottomrule[1.3pt]
\end{tabular}}
\end{table}

\textit{Implementations.} We employ the AE architecture described in~\citet{xie2016unsupervised} for all datasets. The encoder is a fully connected multi-layer perceptron (MLP) with dimensions $D\text{-}500\text{-}500\text{-}2000\text{-}d$. $D$ is the dimension of input. $d$ is the dimension of centroids, which is set to 10 for all datasets. All layers use ReLU activation except the last. The decoder is mirrored of the encoder. \yyh{Compared with those AE-based clustering methods~\citep{xie2016unsupervised,guo2017improved}, our SCAB introduces only one extra linear layer for Eq.\eqref{eq:z_hat_final}, which bring negligible network parameter overhead.}
\yyh{We apply SCAB to raw data for \textit{UCI-Face}, \textit{Rotated Fashion}, \textit{MNIST-USPS}, \revise{\textit{HAR} and \textit{Rotated Fashion-Con}} considering their simplicity. Inspired by the recent state-of-art (SOTA) clustering methods~\citep{tsai2021mice,niu2022spice}, which rely on structured representations to achieve superior performance on complex datasets, we apply SCAB to the extracted features for \textit{Office-31} and \textit{CIFAR10-C} considering their complexity. We use ImageNet-pretrained ResNet50 to extract features for Office-31 following the SOTA clustering method on Office-31~\citep{li2020deep}. We use MoCo~\citep{he2020momentum} to extract features for \textit{CIFAR10-C} following the SOTA clustering method on \textit{CIFAR10-C}~\citep{niu2022spice}.} Note that these feature extractors do not utilize any supervision regarding the datasets. 
We adopt the Adam optimizer. The default learning rate, training epoch, and batch size are $5\text{e}\text{-}4$, $1,000$, and $256$, respectively. 

\textit{Baselines.} The method that removes the confounding factor in the raw space via linear projection, i.e., RUV~\citep{jacob2016correcting} (Eq.\eqref{eq:ruv}, Eq.\eqref{eq:lp}), is included as our first baseline. Further, we extend RUV to eliminate the confounding factor in the latent space. In Particular, we first train AE to obtain the latent representations for \textit{UCI-Face}, \textit{Rotated Fashion}, \textit{MNIST-USPS} and \revise{\textit{HAR}}. We use the extracted features described above as the representations for \textit{Office-31} and \textit{CIFAR10-C}. Then, we apply RUV to remove the bias information from the representations. We name these two baselines as $\text{RUV}_{x}$ and $\text{RUV}_{z}$, respectively. \yyh{We also consider Iterative Spectral Method (ISM)~\citep{wu2019solving}) and Deep Fair Clustering (DFC)~\citep{li2020deep} as our baselines since these two methods can be deemed as the same objective as ours (Eq.\eqref{eq:scab}). We do not compare with other fair clustering methods since they have different goals from our setting (see Section~\ref{sect:ac&fc}). \textit{For a fair comparison, we take raw images of \textit{UCI-Face}, \textit{Rotated Fashion} and \textit{MNIST-USPS} and extracted features of \textit{Office-31} and \textit{CIFAR10-C} as input for all the baselines except for $\text{RUV}_{x}$, which takes raw data as input.}} \revise{ISM, DFC and RUV are designed for the discrete confounding factor and cannot be applied to the continuous one, so they are not run on \textit{Rotated Fashion-Con}.}

\textit{Metrics.} We evaluate different clustering methods with three widely-used clustering metrics,  i.e., accuracy (ACC), normalized mutual information (NMI) and \revise{Adjusted Rand Index (ARI)}. For these metrics, values range between 0 and 1, and a higher value indicates better performance.

\begin{table}[!t]
\centering
\caption{\label{tb:all} SCAB compared with baselines \mbox{that can remove the confounding} factor w.r.t. ACC ($\uparrow$), NMI ($\uparrow$) and \revise{ARI ($\uparrow$)}. The best results are highlighted in bold. The second-best results are underlined. 
}
\begin{tabular}{c|c|c|c|cc|c}
\bottomrule[1.3pt]
Dataset & Metric & ISM & DFC & $\text{RUV}_{x}$ & $\text{RUV}_{z}$ & SCAB \\\hline
\multirow{3}{*}{\textit{UCI-Faces}}       
& ACC & 0.763 & 0.394 & 0.380 & 0.539 & \textbf{0.824} \\
& NMI & 0.454 & 0.087 & 0.163 & 0.322 & \textbf{0.570} \\
& ARI & 0.482 & 0.054 & 0.042 & 0.198 & \textbf{0.554} \\ \hline
\multirow{3}{*}{\textit{Rotated Fashion}} 
& ACC & N.A. & 0.539 & 0.579 & \textbf{0.993} & \underline{0.985} \\
& NMI & N.A. & 0.351 & 0.516 & \textbf{0.969} & \underline{0.940} \\
& ARI & N.A. & 0.248 & 0.318 & \textbf{0.982} & \underline{0.961} \\\hline
\multirow{3}{*}{\textit{MNIST-USPS}}      
& ACC & N.A. & \underline{0.825} & 0.457 & 0.785  & \textbf{0.919} \\
& NMI & N.A. & \underline{0.789} & 0.379 & 0.756  & \textbf{0.837} \\ 
& ARI & N.A. & - & 0.236 & 0.690 & \textbf{0.831} \\\hline
\multirow{3}{*}{\textit{Office-31}}       
& ACC & 0.659 & \underline{0.692} & 0.186 & 0.673  & \textbf{0.724} \\
& NMI & 0.671 & \underline{0.718} & 0.232 & 0.714  & \textbf{0.728} \\ 
& ARI & 0.495 & - & 0.065 & 0.548 & \textbf{0.565} \\\hline
\multirow{3}{*}{\textit{CIFAR10-C}}       
& ACC & N.A. & 0.283 & 0.208 & \underline{0.357}  & \textbf{0.458} \\
& NMI & N.A. & 0.186 & 0.085 & \textbf{0.317} & \underline{0.311}  \\
& ARI & N.A. & 0.105 & 0.040 & \underline{0.087} & \textbf{0.274} \\ \hline
\revise{\multirow{3}{*}{\textit{HAR}}}       
& ACC & 0.556 & 0.722 & \underline{0.732} & 0.715  & \textbf{0.823} \\
& NMI & 0.477 & 0.632 & 0.689 & \underline{0.791} & \textbf{0.830}  \\
& ARI & 0.368 & 0.546 & 0.598 & \underline{0.671} & \textbf{0.754} \\
\bottomrule[1.3pt]
\end{tabular}
\end{table}

\subsection{Performance comparison}
\label{sect:efficacy}
Quantitative results of our SCAB and various baselines that can remove the confounding factor are summarized in Table~\ref{tb:all}. It shows that: 
(1) \textbf{SCAB obtains superior results on all datasets.}  This is because it adopts an effective non-linear dependence measure and a joint training paradigm, which can learn clustering-favorable representations invariant to the confounding factor. 
\revise{(2) SCAB can be applied for removing the continuous confounding factor (see \textit{Rotated Fashion-Con} in Table~\ref{tb:simple}) while existing baselines cannot.}
(3) \textbf{Latent space is better than raw space. Non-linear correlation is better than linear correlation.} $\text{RUV}_{z}$ achieves better performance than $\text{RUV}_{x}$, which shows that removing the confounding factor in the latent space is more effective than that in the raw space. $\text{RUV}_{z}$ obtains worse results than our SCAB on four datasets since $\text{RUV}_{z}$ simply adopts linear projection and heavily relies on the extracted representations beforehand, which cannot deal with these complex datasets where the desired clustering factor and the confounding factor are coupled non-trivially in the latent space.
(4) \textbf{DFC originally designed for two categories degenerates on the dataset with more categories} (i.e., \textit{UCI-Faces}, \textit{Rotated Fashion}, and \textit{CIFAR10-C}). On one hand, more categories may increase the difficulty of adversarial training, making it unable to effectively remove the confounding factor. On the other hand, the constraint requires training a DEC~\citep{xie2016unsupervised} for each category of data. For example, it needs to train a DEC on around 93 images for \textit{UCI-Face}, which would suffer from insufficient training samples. 
(5) \textbf{ISM cannot be executed on large-scale datasets}, i.e, \textit{Rotated Fashion}, \textit{MNIST-USPS} and \textit{CIFAR10-C}. ISM requires a memory complexity of $\mathcal{O}(n^2)$ and needs to store a data matrix with a size larger than $10k \times 10k$ for these datasets, which is beyond our computing capacity. 

\subsection{Efficacy of removing the confounding factor for clustering}

\begin{table}[!t]
\centering
\caption{\label{tb:simple} \revise{SCAB compared with standard clustering w.r.t. ACC ($\uparrow$), NMI ($\uparrow$) and ARI ($\uparrow$) on four simple image datasets and one signal-vector dataset.}}
\begin{tabular}{c|c|cc|c}
\bottomrule[1.3pt]
Dataset & Metric & $k$-means & IDEC & SCAB \\ \hline
\multirow{3}{*}{\textit{UCI-Faces}}       
& ACC  & 0.266 & 0.356  & \textbf{0.824} \\
& NMI  & 0.002 & 0.069  & \textbf{0.570} \\
& ARI  & -0.001 & 0.058  & \textbf{0.554} \\\hline
\multirow{3}{*}{\textit{Rotated Fashion}} 
& ACC & 0.487 & 0.602 & \textbf{0.985} \\
& NMI & 0.414 & 0.611 & \textbf{0.940} \\
& ARI  & 0.260 & 0.465  & \textbf{0.961} \\\hline
\multirow{3}{*}{\textit{MNIST-USPS}}      
& ACC & 0.506 & 0.789 & \textbf{0.919} \\
& NMI & 0.447 & 0.766 & \textbf{0.837} \\
& ARI & 0.333 & 0.689  & \textbf{0.831} \\ \hline
\revise{\multirow{3}{*}{\textit{HAR}}}
& ACC & 0.600 & 0.680 & \textbf{0.823} \\
& NMI & 0.589 & 0.733 & \textbf{0.830} \\
& ARI & 0.461 & 0.632 & \textbf{0.754} \\ \hline
\revise{\multirow{3}{*}{\textit{Rotated Fashion-Con}}}     
& ACC & 0.369 & 0.387 & \textbf{0.576} \\
& NMI & 0.228 & 0.287 & \textbf{0.399} \\
& ARI  & 0.139 & 0.191  & \textbf{0.329} \\
\bottomrule[1.3pt]
\end{tabular}
\end{table}

To demonstrate the gain of clustering that takes into account the removal of the confounding factor, we include the comparison with standard clustering methods -- $k$-means~\citep{bishop2006pattern}, IDEC~\citep{guo2017improved}\footnote{IDEC is a representative AE-based clustering method.}, PICA~\citep{DBLP:conf/cvpr/HuangGZ20} and SPICE~\citep{niu2022spice}\footnote{PICA and SPICE are recently proposed self-supervised clustering methods. SPICE is the SOTA method.} in Table~\ref{tb:simple} and Table~\ref{tb:complex}. We apply PICA and SPICE only on \textit{Office-31} and \textit{CIFAR10-C} considering that they were proposed for complex image datasets. For a fair comparison, we take raw images of \textit{UCI-Face}, \textit{Rotated Fashion}, \textit{MNIST-USPS}, \revise{\textit{HAR} and \textit{Rotated Fashion-Con}} and extracted features of \textit{Office-31} and \textit{CIFAR10-C} as input for the methods except for PICA. PICA takes raw images of all datasets as input since it needs to conduct image augmentations for partition confidence maximization~\citep{DBLP:conf/cvpr/HuangGZ20}.

\paragraph{Improved by removing the confounding factor}
Table~\ref{tb:simple} and Table~\ref{tb:complex} show that: compared with standard clustering methods, our SCAB achieves superior performance on all datasets. It verifies the claim that our SCAB which explicitly removes the influence of the confounding factor performs better than the standard clustering methods. Note that PICA obtains poor results since it conducts clustering on raw features ($k$-means on MoCo extracted feature achieves better results than PICA on raw features also reported in~\citet{tsai2021mice}). And SPICE performs worse than IDEC because it applies a discriminative model for clustering, which is more vulnerable to the confounding factor than IDEC which is AE-based clustering.

\begin{figure}[!tb]
    \centering
    \begin{minipage}[t]{.136\linewidth}
    \includegraphics[width=\linewidth]{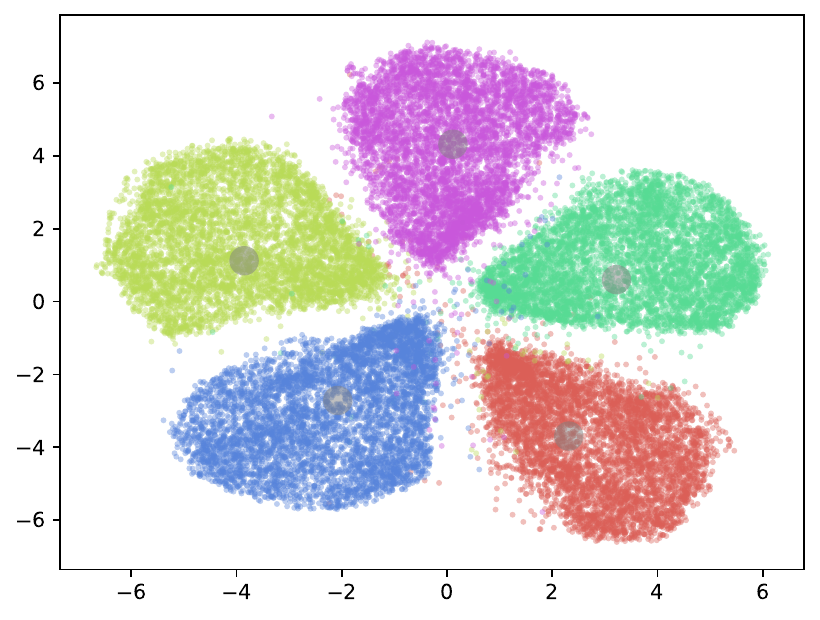}
  \end{minipage}
  \begin{minipage}[t]{.136\linewidth}
    \includegraphics[width=\linewidth]{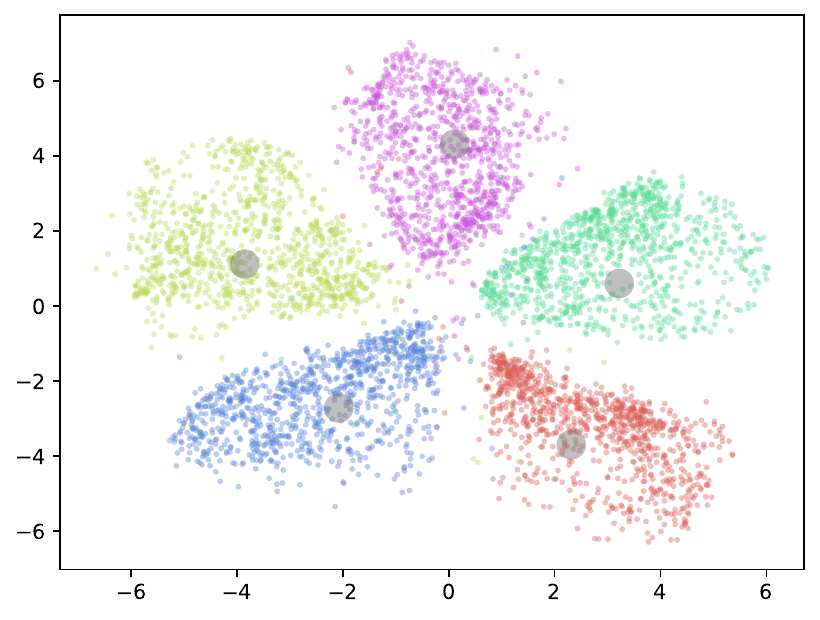}
  \end{minipage}
  \begin{minipage}[t]{.136\linewidth}
    \includegraphics[width=\linewidth]{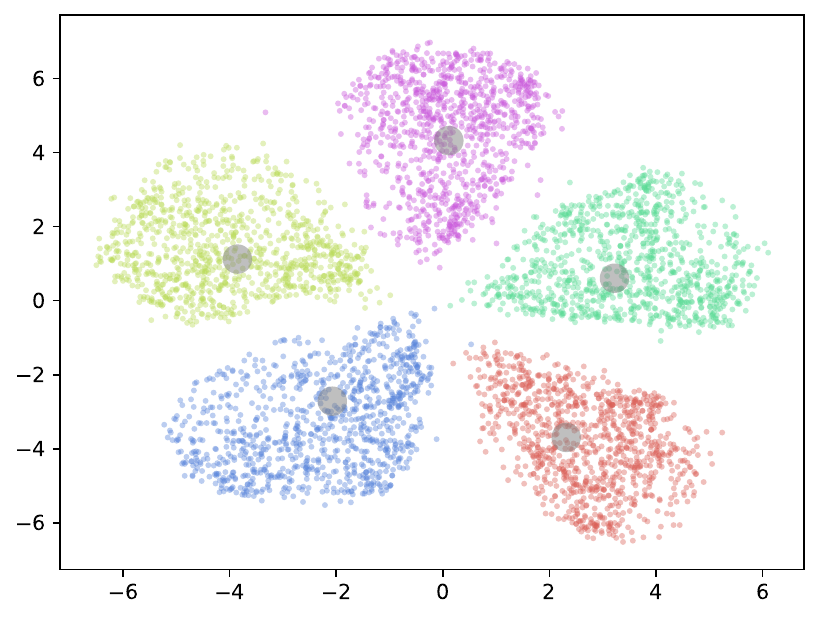}
  \end{minipage}
  \begin{minipage}[t]{.136\linewidth}
  \centering
    \includegraphics[width=\linewidth]{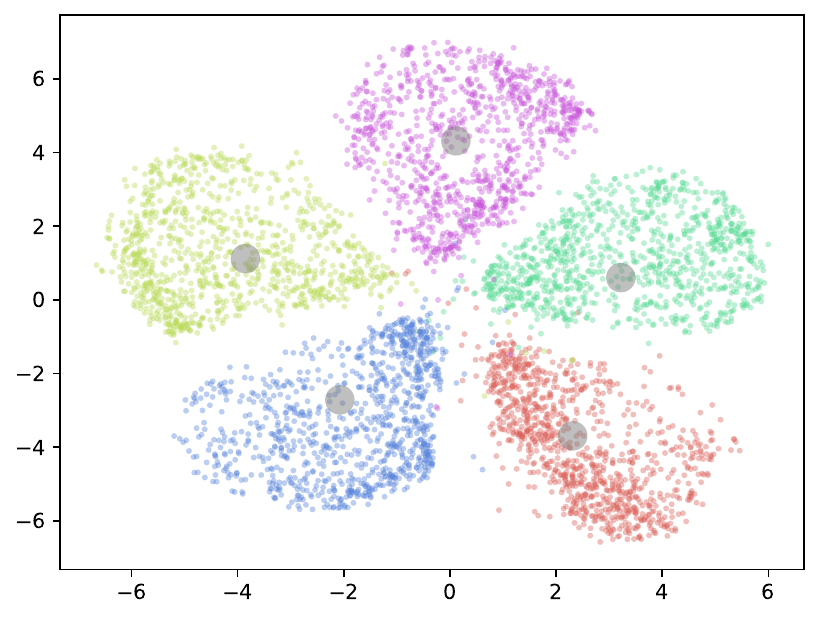}
  \end{minipage}
  \begin{minipage}[t]{.136\linewidth}
    \includegraphics[width=\linewidth]{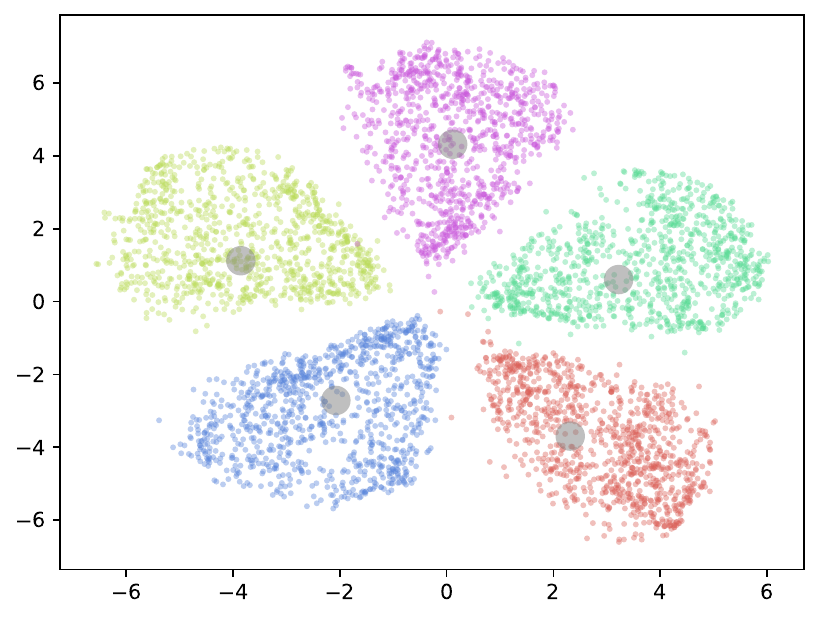}
  \end{minipage}
  \begin{minipage}[t]{.136\linewidth}
    \includegraphics[width=\linewidth]{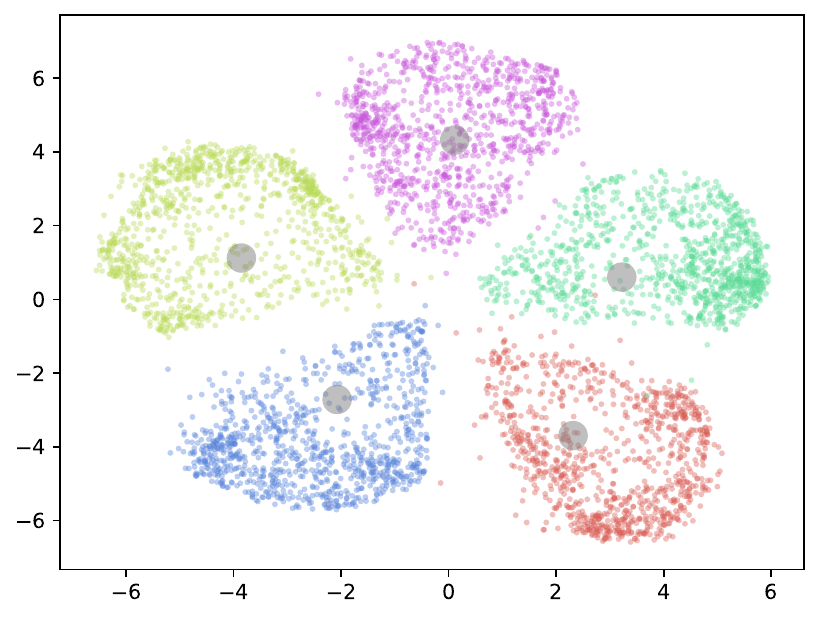}
  \end{minipage}
  \begin{minipage}[t]{.136\linewidth}
    \includegraphics[width=\linewidth]{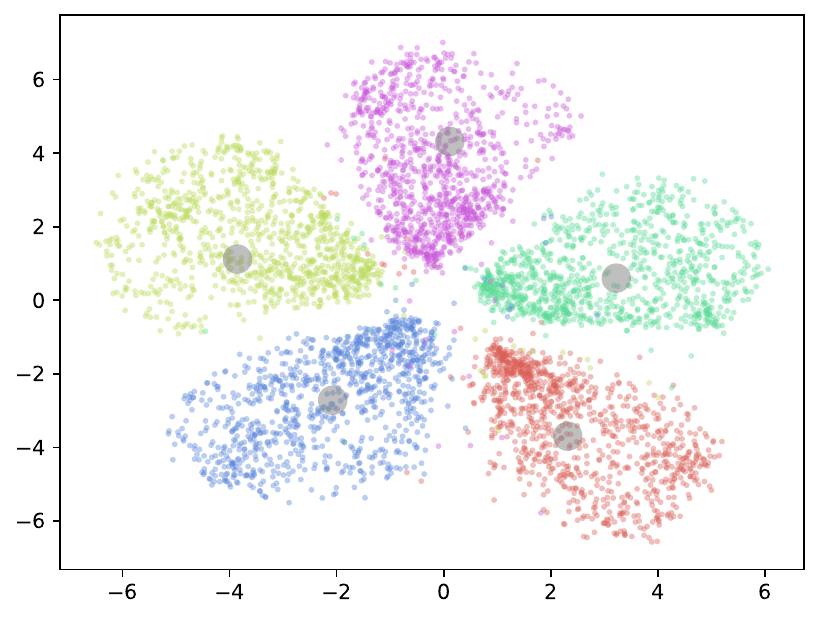}
  \end{minipage}
      \begin{minipage}[t]{.136\linewidth}
  \centering
    \includegraphics[width=\linewidth]{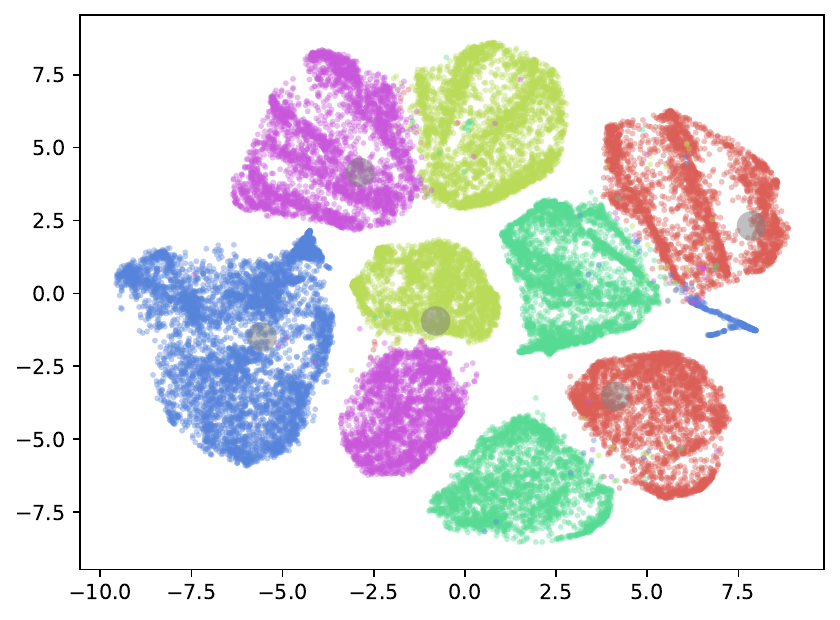}
    \centerline{All data\label{fig:fmnist_tsne_all_ae_kmeans}}
  \end{minipage}
  \begin{minipage}[t]{.136\linewidth}
  \centering
    \includegraphics[width=\linewidth]{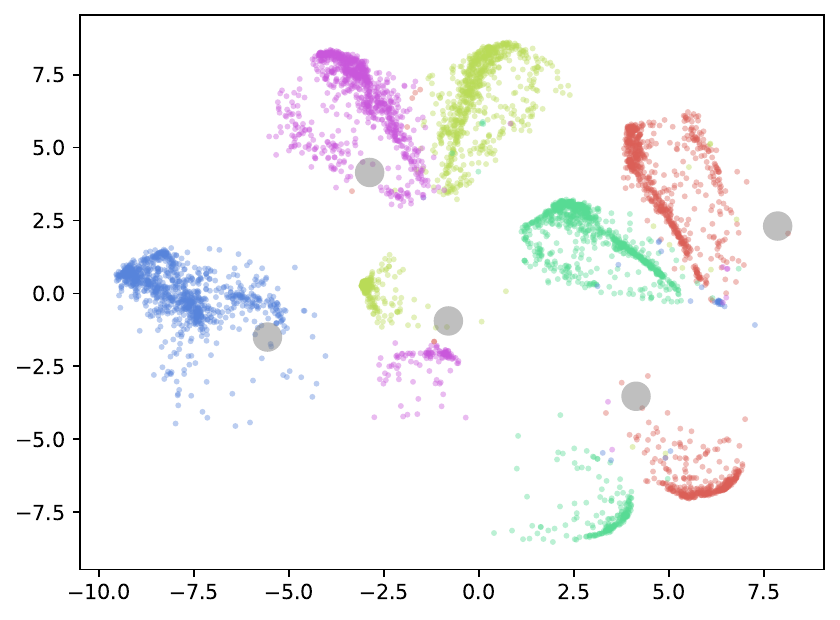}
    \centerline{T-shirt/top\label{fig:fmnist_tsne_c_ae_kmeans}}
  \end{minipage}
  \begin{minipage}[t]{.136\linewidth}
  \centering
    \includegraphics[width=\linewidth]{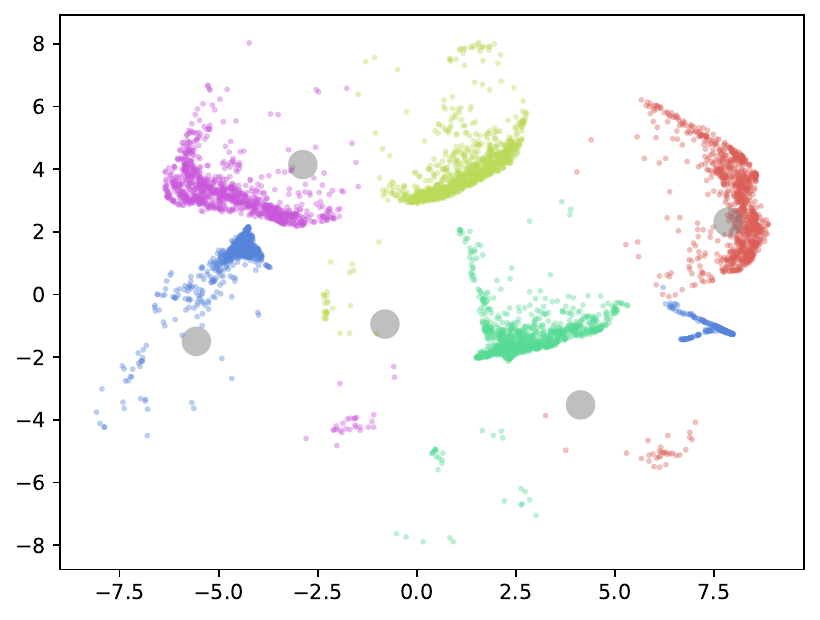}
    \centerline{Trouser}
  \end{minipage}
  \begin{minipage}[t]{.136\linewidth}
  \centering
    \includegraphics[width=\linewidth]{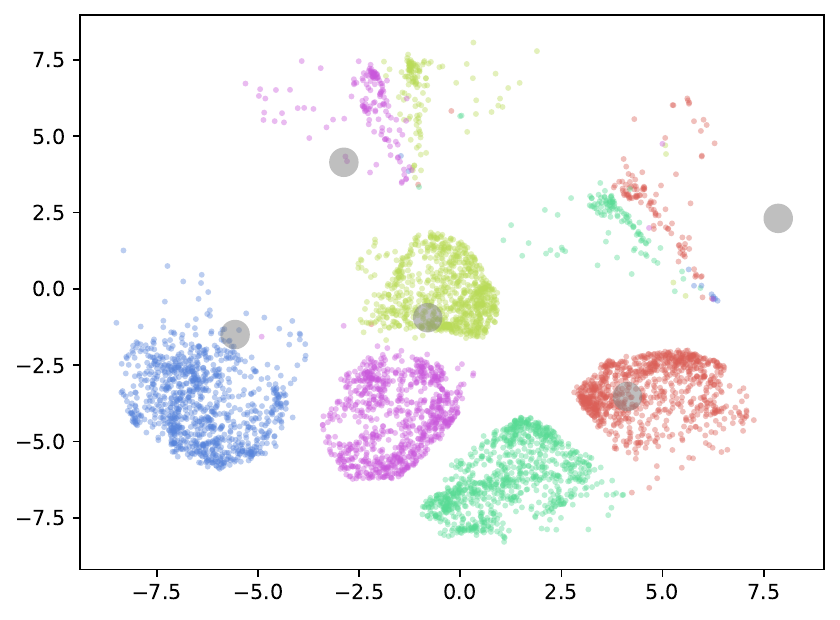}
    \centerline{Pullover}
  \end{minipage}
  \begin{minipage}[t]{.136\linewidth}
  \centering
    \includegraphics[width=\linewidth]{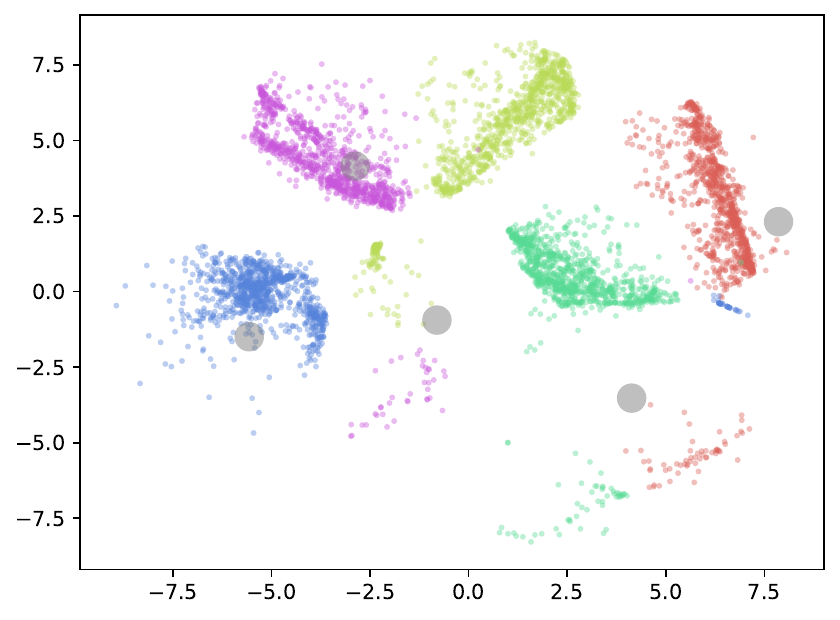}
    \centerline{Dress}
  \end{minipage}
  \begin{minipage}[t]{.136\linewidth}
  \centering
    \includegraphics[width=\linewidth]{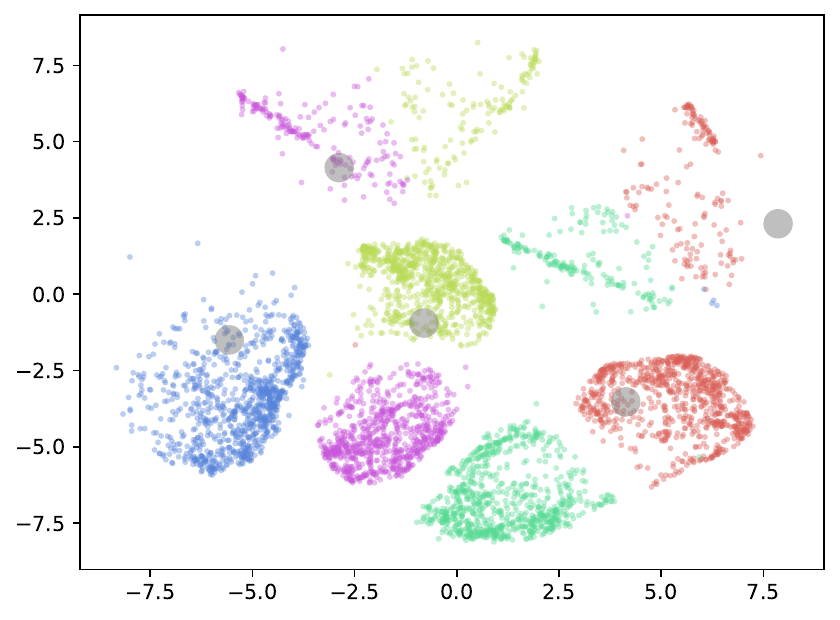}
    \centerline{Coat}
  \end{minipage}
  \begin{minipage}[t]{.136\linewidth}
  \centering
    \includegraphics[width=\linewidth]{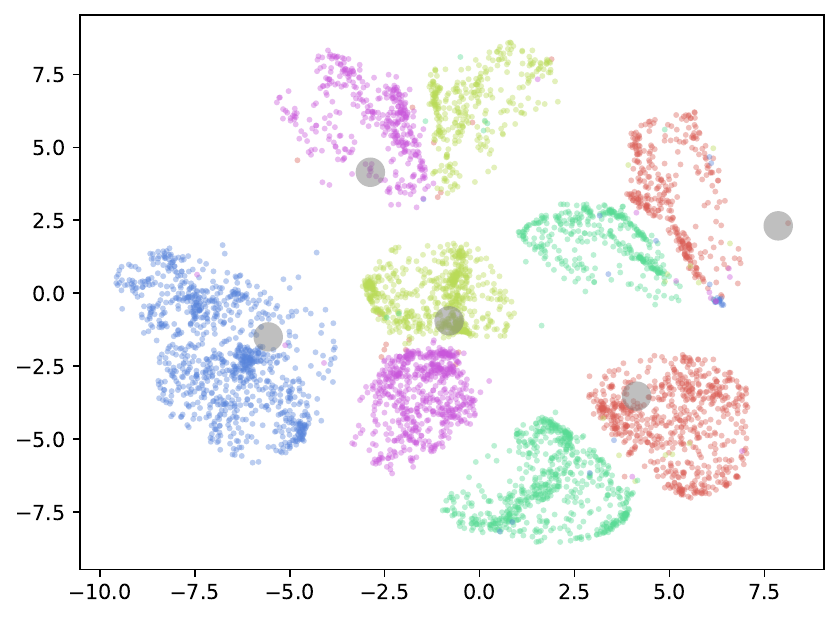}
    \centerline{Shirt\label{fig:fmnist_tsne_h_idec}}
  \end{minipage}
    \caption{t-SNE on latent representations and clustering centroids from SCAB (1st raw) and IDEC (2nd raw) on \textit{Rotated Fashion}, respectively. The big grey dots are the centroids. The small dots are the representations, of which the colors denote the ground truth category labels. \label{fig:fmnist}} 
\end{figure}

\begin{table}[!tb]
\centering
\caption{\label{tb:complex} SCAB compared with standard clustering w.r.t. ACC ($\uparrow$), NMI ($\uparrow$) and \revise{ARI ($\uparrow$)} on two complex image datasets.} 
\begin{tabular}{c|c|cccc|c}
\bottomrule[1.3pt]
Dataset & Metric & $k$-means & IDEC & PICA & SPICE & SCAB \\\hline
\multirow{3}{*}{\textit{Office-31}}       
& ACC & 0.648 & 0.634 & 0.440 & 0.231 & \textbf{0.724} \\
& NMI & 0.689 & 0.690 & 0.536 & 0.341 & \textbf{0.728} \\
& ARI & 0.506 & 0.500 & 0.305 & 0.117 & \textbf{0.565} \\\hline
\multirow{3}{*}{\textit{CIFAR10-C}}       
& ACC & 0.247 & 0.420 & 0.220 & 0.313 & \textbf{0.458} \\
& NMI & 0.225 & \textbf{0.380} & 0.178 & 0.294 & {0.311}  \\
& ARI & 0.074 & 0.257  & 0.082 & 0.149 & \textbf{0.274} \\
\bottomrule[1.3pt]
\end{tabular} 
\end{table}

\begin{figure}[!ht]
  \begin{minipage}[t]{.47\linewidth}
  \centering
    \includegraphics[width=\linewidth]{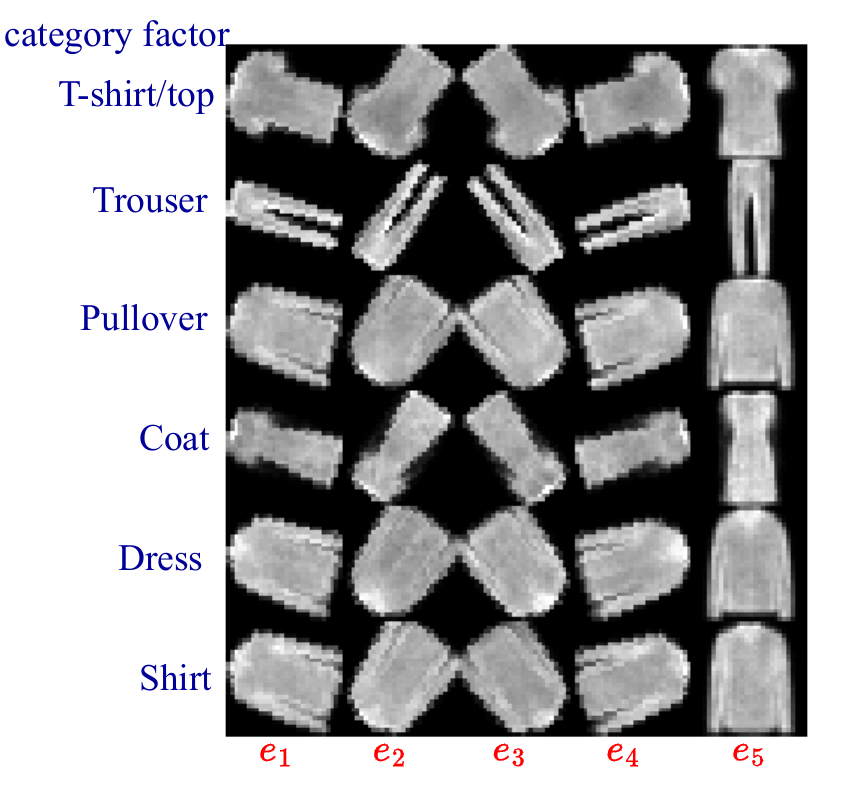}\vskip-0.05in
    \centerline{(a) \textit{Rotated Fashion} ($28\times 28$)}
  \end{minipage}
    \begin{minipage}[t]{.515\linewidth}
  \centering
    \includegraphics[width=\linewidth]{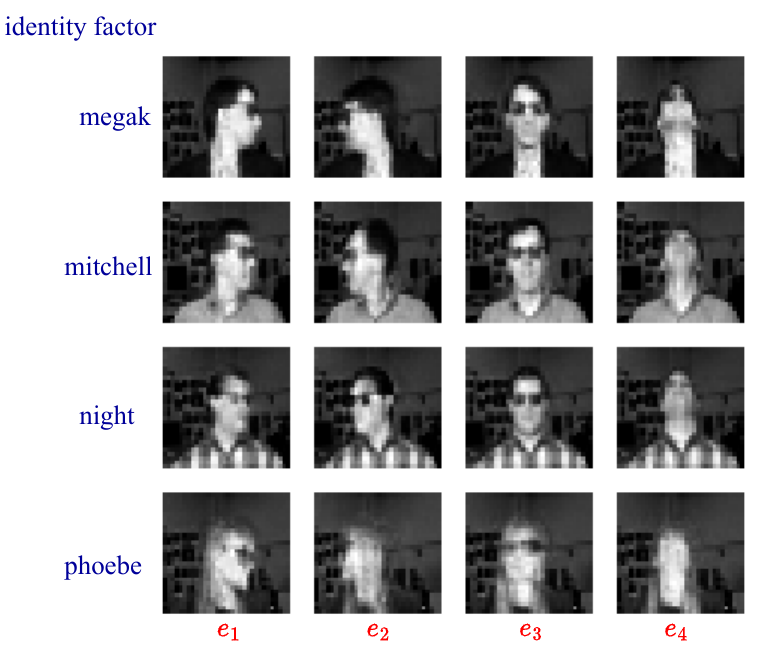}\vskip-0.05in
    \centerline{(b) \textit{UCI-Face} ($32\times 30$)}
  \end{minipage}
    \begin{minipage}[t]{1\linewidth}
     \centering
    \includegraphics[width=\linewidth]{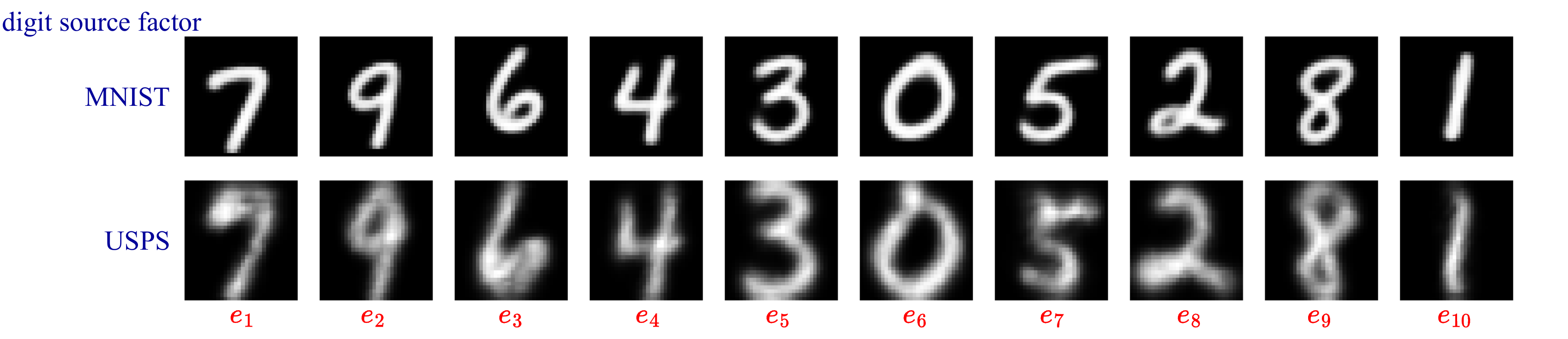}\vskip-0.05in
    \centerline{(c) \textit{MNIST-USPS} ($32\times 32$)}
  \end{minipage}
    \caption{\label{fig:cen_scab} Centroids' reconstruction of SCAB on \textit{Rotated Fashion}, \textit{UCI-Face} and \textit{MNIST-USPS}, respectively. Each column is conditioned on the same clustering centroid. Each row is conditioned on different labels of the cloth category factor, the identity factor, and the digit source factor, respectively.}
\end{figure}

\paragraph{Invariant representations} To further illustrate the effectiveness of removing the confounding factor, we visualize the latent representations and the clustering centroids for our SCAB and IDEC (i.e., standard clustering that ignores the confounding factor) on \textit{Rotated Fashion}, respectively. From the t-SNE visualization of our SCAB (the first row of Fig.~\ref{fig:fmnist}), we can see that: (1) the clusters are well separated and the centroids are located at the center of each cluster. (2) These categories’ representations are not only well aligned with each other, but also the whole data’s representations. This demonstrates that our SCAB's latent representations are invariant to the confounding factor, i.e., the cloth category label. (3) Each centroid represents one of the five rotation angles in the dataset. In addition, the reconstruction of the centroids is exactly the Fashion-MNIST objects, which demonstrates our SCAB captures semantic clustering structures. 

The t-SNE visualization of IDEC (the second row of Fig.~\ref{fig:fmnist}) shows that: (1) IDEC obtains an inferior clustering structure due to the negative impact of the confounding factor. Specifically, the cloth category introduces variances into the data, making the derived structure away from the desired one w.r.t. the rotation factor. (2) These categories' representations are neither aligned with each other nor with the representation of the entire data. It demonstrates that IDEC's latent representations are corrupted by the confounding factor, i.e., variances of cloth category.

\paragraph{Disentangled centroid reconstruction} 

We can reconstruct the centroids conditioned on the confounding factor for SCAB.
Fig.~\ref{fig:cen_scab} shows that (1) the latent embedding~$z$ and the confounding factor~$c$ are well disentangled. In particular, the information of the confounding factor is well captured by~$c$. (2) The centroids can capture clear structures, i.e., rotation angles for \textit{Rotated Fashion}, the pose angle for \textit{UCI-Face}, and the digit type for \textit{MNIST-USPS}, respectively. On \textit{Office-31} and \textit{CIFAR10-C}, we do not reconstruct the centroids on these datasets as the extracted features are used as model input.
\begin{wrapfigure}{r}{0.45\textwidth} \vskip-0.25in
  \centering
  \includegraphics[width=\linewidth]{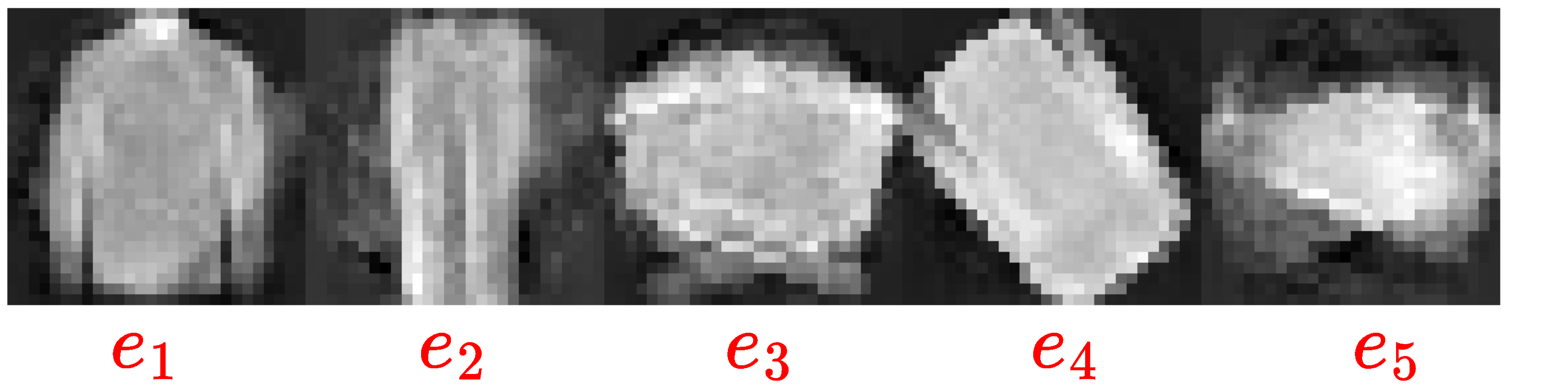} \vskip-0.05in
   \caption{\label{fig:cen_idec} Centroids' reconstruction of IDEC on \textit{Rotated Fashion}.} \vskip-0.2in
\end{wrapfigure}

Fig.~\ref{fig:cen_idec} shows that (1) IDEC does not have the ability to disentangle the confounding factor~$c$ from the latent space. (2) Its centroids do not capture all rotation angles in the dataset as they are distracted by the cloth categories. For example, $e_1$ and $e_2$ represent the shirt and the trouser with the same angle, respectively. 

\subsection{Ablation study}
We study the effectiveness of each module by excluding it from our SCAB (Fig.~\ref{fig:network}). 
\begin{table}[!htb]
\centering
\caption{\label{tb:ablation} Ablation study of SCAB on \textit{Rotated Fashion}. ``Clu'' means the clustering module. ``Dis'' means the disentanglement module.} 
\begin{tabular}{c|c|c|c|c}
\toprule[1.3pt]
Metric  & w/o Clu & w/o Dis  & w/o Clu \& Dis & SCAB \\ \hline
ACC  & 0.513 & 0.857 & 0.487 & \textbf{0.985}    \\ 
NMI  & 0.376  & 0.803 & 0.414 & \textbf{0.940}    \\
ARI  & 0.277  & 0.757 & 0.260 & \textbf{0.961}    \\
\bottomrule[1.3pt]
\end{tabular} 
\end{table}

Table~\ref{tb:ablation} shows that: (1) our SCAB gets the best results, which justifies the necessity of each module. (2) Without the disentanglement module to remove the confounding factor via mutual information, the clustering performance drops significantly since the confounding factor would distract desired the clustering results. (3) A poor clustering structure is obtained without the clustering module because it fails to derive clustering-friendly representations.
(4) The clustering performance is worse when excluding both the clustering module and the disentanglement module.

\subsection{Extension to the incomplete confounding factor}\label{sect:semi}

We explore the performance of SCAB given different amounts of labeled data w.r.t. the confounding factor on \textit{Rotated Fashion}. Applying SCAB to this semi-supervised setting, we first train a classifier on the labeled data and use it to predict labels for the remaining unlabeled data. Then SCAB is naturally applied to these fully-labeled data. Particularly, we employ a convolutional neural network classifier for the classification. IDEC is adopted as the baseline without removing the confounding factor following the same setting as SCAB.

We plot the test accuracy of the classifier (calculated on the remaining unlabeled data) and the clustering performance (ACC and NMI) of SCAB in Fig.~\ref{fig:semi} with the percentage of labeled data from 0.1\% to 100\%. 
It shows that 
(1) compared to IDEC which ignores the confounding factor, our SCAB can improve the clustering performance even with a very small amount of labeled data. 
\begin{wrapfigure}{r}{0.55\textwidth} \vskip-0.2in
  \centering
  \includegraphics[width=1\linewidth]{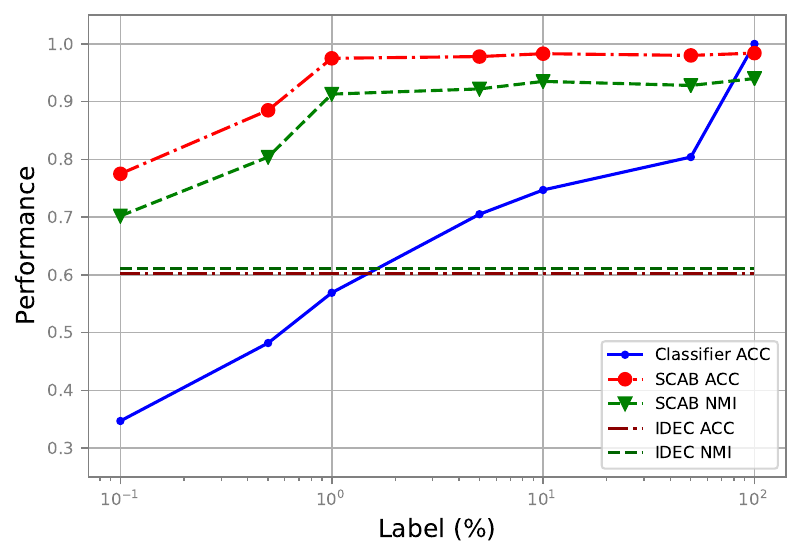} \vskip-0.05in
    \caption{\label{fig:semi}Clustering performance of SCAB given partial labels w.r.t. the confounding factor on \textit{Rotated Fashion}. ``Classifier ACC'' is the test accuracy of the classifier. $x$~axis is the ratio of labeled data. } \vskip-0.1in
\end{wrapfigure}
(2) When there are less than 0.5\% labeled data, the test accuracy of the classifier is low, smaller than 0.5. Accordingly, the results of SCAB are relatively not so good since there are more than 50\% samples assigned with wrong labels. 
(3) When the labeled data is larger than 1\%, there are more than 50\% samples assigned with true labels. 
Though the percentage of label noise is still very high, SCAB can perform well since the correct labels dominate and the structured representations can be robust to label noise. In conclusion, our SCAB can work well even given a small amount of labeled data regarding the confounding factor.

\section{Conclusion}
\label{sect:con}
We have introduced a general framework SCAB for a new stream of clustering that aims to deliver clustering results invariant to the pre-designated confounding factor. 
SCAB is the first deep clustering framework that can eliminate the confounding factor in the semantic latent space of complex data via a non-linear dependence measure with theoretical guarantees.
We have demonstrated the efficacy of SCAB on various datasets using label indicators of the confounding factor.
In the future, we can extend our SCAB to more types of data, e.g., text/ time series data. In addition, while this study focuses on sanitized clustering given the known confounding factor with (partially) labeled supervision, it is interesting to explore clustering with unindicated confounding factors. \revise{Last, theoretical analysis on the confounding factor that is not fully observed is also a potential direction.}


\bibliography{sn-bibliography}

\end{document}